\documentclass[letterpaper]{article}

\usepackage{microtype}
\usepackage{graphicx}
\usepackage{booktabs} % for professional tables
\usepackage{amsfonts,amsthm,amsmath}
\usepackage{xspace}
\usepackage{pifont}% http://ctan.org/pkg/pifont
\usepackage[numbers]{natbib}
\usepackage{uai2020}
\usepackage[margin=1in]{geometry}
\usepackage{xcolor}
\usepackage{times}
\usepackage{algorithm}
\usepackage{algorithmicx}
\usepackage{url}

\usepackage[noend]{algpseudocode}
\algdef{SE}[SUBALG]{Indent}{EndIndent}{}{\algorithmicend\ }%
\algtext*{Indent}
\algtext*{EndIndent}
\title{Symbolic Querying of Vector Spaces: \\
Probabilistic Databases Meets Relational Embeddings}

\author{{\bf Tal Friedman} \\
Department of Computer Science \\
University of California, Los Angeles \\
\texttt{tal@cs.ucla.edu}
\And
{\bf Guy Van den Broeck} \\
Department of Computer Science \\
University of California, Los Angeles \\
\texttt{guyvdb@cs.ucla.edu}} % LEAVE BLANK FOR ORIGINAL SUBMISSION.
\newcommand{\pdb}{\mathcal{P}}
\newcommand{\rf}{\mathcal{R}}
\newcommand{\ef}{\mathcal{E}}

\newcommand{\df}{\mathcal{D}}

\newcommand{\kb}{\mathcal{K}}

\newcommand{\Query}{\ensuremath{\mathsf{Q}}\xspace}

\newcommand{\ax}[1]{\ensuremath{\left<#1\right>}\xspace}

\newcommand{\TO}{\textsc{TractOR}}
\newcommand{\TOP}{\textsc{TractOR+}}
\newcommand{\TOs}{\TO\ }
\newcommand{\TOPs}{\TOP\ }

\newtheorem{theorem}{Theorem}
\theoremstyle{definition}
\newtheorem{definition}{Definition}
\begin{document}

\maketitle
\begin{abstract}
We propose unifying techniques from probabilistic databases and relational embedding models with the goal of performing complex queries on incomplete and uncertain data. We formalize a probabilistic database model with respect to which all queries are done. This allows us to leverage the rich literature of theory and algorithms from probabilistic databases for solving problems. While this formalization can be used with any relational embedding model, the lack of a well-defined joint probability distribution causes simple query problems to become provably hard. With this in mind, we introduce \TO, a relational embedding model designed to be a tractable probabilistic database, by exploiting typical embedding assumptions within the probabilistic framework. Using a principled, efficient inference algorithm that can be derived from its definition, we empirically demonstrate that \TOs is an effective and general model for these querying tasks.

%   We study the problem of performing complex queries in the setting of relational embedding models, beyond the standard task of link prediction. Our approach uses techniques from probabilistic databases to define a single probabilistic model with respect to which all queries and inference tasks are done. This allows us to leverage the rich literature of theory and algorithms from probabilistic databases for solving problems. While this formalization can be used with any relational embedding model, the lack of a well defined joint probability distribution causes simple problems to become provably hard. With this in mind, we introduce \emph{TractOR}, a relational embedding model designed in terms of probabilistic databases to facilitate efficient inference for a much broader class of queries. Using principled, efficient inference algorithms that can be derived from its definition, we empirically demonstrate that \emph{TractOR} is an effective and general model for these tasks.
\end{abstract}

\section{INTRODUCTION}
Relational database systems are ubiquitous tools for data management due to their ability to answer a wide variety of queries. In particular, languages such as SQL allow one to take advantage of the relational structure of the data to ask complicated question to learn, analyse, and draw conclusions from data. However, traditional database systems are poorly equipped to deal with uncertainty and incompleteness in data. Meanwhile, techniques from the machine learning community can successfully make predictions and infer new facts. In this work we marry ideas from both machine learning and databases to provide a framework for answering such queries while dealing with uncertain and incomplete data.

The first key question we need an answer for when dealing with uncertain relational data is how to handle the fact that our data is invariably incomplete. That is, there will always be facts that we do not explicitly see, but would like to be able to infer. In the machine learning community, this problem is known as \emph{link prediction}, a task which has garnered a lot of attention in recent years~\citep{Nickel2015ARO,Nickel2011ATM,Kazemi2018SimplEEF, Trouillon2016ComplexEF} using a variety of techniques~\citep{Blockeel1997TopdownIO,Dumancic2019ACS}. Recently, the most common techniques for this problem are relational embedding models, which embed relations and entities as vectors and then use a scoring function to predict whether or not facts are true. While these techniques are popular and have proven effective for link prediction, they lack a consistent underlying probabilistic semantics, which makes their beliefs about the world unclear. As a result, investigations into them have rarely gone beyond link prediction~\citep{Hamilton2018EmbeddingLQ,Krompass2014QueryingFP}.

On the other hand, the databases community has produced a rich body of work for handling uncertainty via probabilistic databases (PDBs). In contrast to relational embedding models which are fundamentally predictive models, PDBs \citep{Suciu:2011:PD:2031527,VdBFTDB17} are defined by a probabilistic semantics, with strong and clearly specified independence assumptions. With these semantics, PDBs provide us with a wealth of theoretical and algorithmic research into complex queries, including tractability results \citep{Dalvi2004EfficientQE, Dalvi2007TheDO, Dalvi2012TheDO, Fink2016DichotomiesFQ} and approximations~\citep{Heuvel2019AnytimeAI,Gribkoff2016SlimShotIP}. Recently there has even been work in finding explanations for queries \citep{Ceylan2017MostPE, gribkoff2014most}, and querying subject to constraints \citep{Borgwardt2017OntologyMediatedQF, Bienvenu2016OntologyMediatedQA, FriedmanIJCAI19}. Where PDBs fall short is in two major areas. Firstly, populating PDBs with meaningful data in an efficient way remains a major challenge, due to their brittleness to incomplete data, and due to their disconnect from the statistical models that can provide these databases with probability values. Secondly, while querying is well understood, certain types of desirable queries are provably hard under standard assumptions~\citep{Dalvi2012TheDO}.

In this work, our goal will be to unify the predictive capability of relational embedding models with the sound underlying probabilistic semantics of probabilistic databases. The central question then becomes how should we do this unification such that we maintain as many of the benefits of each as possible, while finding ways to overcome their limitations. As we will discover in Section~\ref{sec:relpdb}, this is not a question with an obvious answer. The straightforward option is to simply convert the relational embedding model's prediction into probabilities, and then use these to populate a probabilistic database. While this does give us a meaningful way to populate a PDB, the resulting model is making some clearly problematic independence assumptions, and moreover still struggles with making certain queries tractable.

At its core, the reason this straightforward solution is ineffective is as follows: while both PDBs and relational embedding models make simplifying assumptions, these assumptions are not being taken into account \emph{jointly}. Each is treating the other as a black box. To overcome this, we incorporate the factorization assumption made by many relational embedding models \citep{Yang2014EmbeddingEA,Nickel2011ATM} directly into our probabilistic database. The resulting model, which we call \TO, thus takes advantages of the benefits of both: it can efficiently and accurately predict missing facts, but it also provides a probabilistic semantics which we can use for complex probabilistic reasoning. Due to its factorization properties, \TOs can even provide efficient reasoning where it was previously difficult in a standard PDB.

The rest of the paper is organized as follows. Section~\ref{sec:pdb} provides the required technical background on PDBs and their associated queries. In Section~\ref{sec:relpdb} we discuss using (tuple-independent) PDBs as the technical framework for relational embedding models, as well as giving a brief formalization and discussion of challenges. Then, in Section~\ref{sec:tractor} we introduce \TO, a relational embedding model designed around PDBs to allow for a large range of efficient queries. Section~\ref{sec:eval} provides an empirical evaluation of \TOs.  Finally, Section~\ref{sec:disc} gives a broad discussion on related work along with ties to future work.

\section{PROBABILISTIC DATABASES} \label{sec:pdb}
We now provide the necessary technical background on probabilistic databases, which will serve as the foundation for our probabilistic semantics and formalism for queries, as well as the underlying inspiration for \TO.
\subsection{RELATIONAL LOGIC AND DATABASES}
We begin with necessary background from \emph{function-free finite-domain} first-order logic. An atom $R(x_1,x_2,...,x_n)$ consists of a predicate $R$ of arity $n$, together with $n$ arguments. These arguments can either be \emph{constants} or \emph{variables}. A \emph{ground atom} is an atom that contains no variables. A \emph{formula} is a series of atoms combined with conjunctions ($\land$) or disjunctions ($\lor$), and with quantifiers $\forall, \exists$. A \emph{substitution} $\Query[x/t]$ replaces all occurrences of $x$ by $t$ in a formula $\Query$.

A relational \emph{vocabulary} $\sigma$ is composed of a set of predicates $\rf$ and a domain $\df$. Using the \emph{Herbrand semantics} \citep{hinrichs2006herbrand}, the \emph{Herbrand base} of $\sigma$ is the set of all ground atoms possible given $\rf$ and $\df$. A $\sigma$-interpretation $\omega$ is then an assignment of truth values to every element of the Herbrand base of $\sigma$. We say that $\omega$ is a \emph{model} of a formula $\Query$ whenever $\omega$ satisfies $\Query$. This is denoted by $\omega \models \Query$.

Under the standard model-theoretic view \citep{Abiteboul1995FoundationsOD}, a relational database for a vocabulary $\sigma$ is a $\sigma$-interpretation $\omega$. In words: a relational database is a series of relations, each of which corresponds to a predicate. These are made up by a series of rows, also called \emph{tuples}, each of which corresponds to a ground atom being true. Any atom not appearing as a row in the relation is considered to be \textit{false}, following the closed-world assumption \citep{reiter1981closed}. Figure~\ref{fig:exdb} shows an example database.

\subsection{PROBABILISTIC DATABASES}
To incorporate uncertainty into relational databases, \emph{probabilistic databases} assign each tuple a probability~\citep{Suciu:2011:PD:2031527,VdBFTDB17}.

\begin{definition}
 A \emph{(tuple-independent) probabilistic database} (PDB) $\pdb$ for a vocabulary $\sigma$ is a finite set of tuples of the form $\langle{t: p}\rangle$ where $t$ is a $\sigma$-atom and $p\in [0,1]$. Furthermore, each $t$ can appear at most once. 
\end{definition}

Given such a collection of tuples and their probabilities, we are now going to define a \emph{distribution} over relational databases. The semantics of this distribution are given by treating each tuple as an independent random variable.

\begin{definition}
  A PDB $\pdb$ for vocabulary $\sigma$ induces a probability distribution over $\sigma$-interpretations $\omega$:
  \begin{align*}
     & P_\pdb(\omega) = \prod_{t \in \omega} P_\pdb(t) \prod_{t \notin \omega} (1-P_\pdb(t))\\
     & \text{ where~~~}  P_\pdb(t) = 
      \begin{cases}
      p & \textnormal{if } \langle t:p \rangle \in \pdb \\
      0 & \textnormal{otherwise}
      \end{cases}
  \end{align*}
\end{definition}

Each tuple is treated as an independent Bernoulli random variable, so the probability of a relational database instance is given as a simple product, based on which tuples are or are not included in the instance.
\begin{figure}[tb]
\centering
\setlength{\tabcolsep}{3pt}
\begin{tabular}{l}
\toprule
 \multicolumn{1}{c}{Scientist} \\  
\midrule
Einstein \\
Erd\H{o}s \\
von Neumann \\
\bottomrule
\end{tabular}
\qquad
\setlength{\tabcolsep}{3pt}
\begin{tabular}{l l}
\toprule
 \multicolumn{2}{c}{CoAuthor}  \\
\midrule
 Einstein & Erd\H{o}s \\
 Erd\H{o}s & von Neumann \\
\bottomrule
\end{tabular}
\caption{Example relational database. Notice that the first row of the right table corresponds to the atom CoAuthor(Einstein, Erd\H{o}s). \label{fig:exdb}}
% \vspace{-5mm}
\end{figure}

\begin{figure}[tb]
\centering
\setlength{\tabcolsep}{3pt}
\scalebox{0.85}{
\begin{tabular}{l  c}
\toprule
 \multicolumn{1}{c}{Scientist} & $\Pr$   \\
\midrule
Einstein & 0.8 \\
Erd\H{o}s & 0.8 \\
von Neumann & 0.9 \\
Shakespeare & 0.2 \\
\bottomrule
\end{tabular}
}
\qquad
\setlength{\tabcolsep}{3pt}
\scalebox{0.85}{
\begin{tabular}{l l  c}
\toprule
 \multicolumn{2}{c}{CoAuthor} & $\Pr$ \\
\midrule
 Einstein & Erd\H{o}s & 0.8 \\
 Erd\H{o}s & von Neumann & 0.9 \\
 von Neumann & Einstein & 0.5 \\
\bottomrule
\end{tabular}
}
\caption{Example probabilistic database. Tuples are now of the form $\langle t:p \rangle$ where $p$ is the probability of the tuple $t$ being present. These tuples are assumed to be independent, so the probability both Einstein and Erd\H{o}s are scientists is $0.8\cdot 0.8 = 0.64$. \label{fig:expdb}}
\vspace{-3mm}
\end{figure}

\subsection{PROBABILISTIC QUERIES}
Much as in relational databases, in probabilistic databases we are interested in answering queries -- the difference being that we are now interested in probabilities over queries. In particular, we study the theory of queries that are fully quantified and with no free variables or cosntants, also known as fully quantified \emph{Boolean queries} -- we will see later how other queries can be reduced to this form. On a relational database, this corresponds to a fully quantified query that has an answer of True or False. %Queries containing constants can be converted to fully quantified Boolean queries in a process known as shattering \citep{VdBFTDB17}.

For example, on the database given in Figure~\ref{fig:exdb}, we might ask if there is a scientist who is a coauthor:
$$ \Query_1 = \exists x. \exists y. S(x) \land \mathit{CoA}(x,y)$$

Which there clearly is, by taking $x$ to be Einstein and $y$ to be Erd\H{o}s. If we instead asked this query of the PDB in Figure~\ref{fig:expdb}, we would be computing the probability by summing over the worlds in which the query is true:
$$ P_\pdb(\Query_1) = \sum_{\omega \models \Query_1} P_\pdb(\omega) $$

Queries of this form that are a conjunction of atoms are called \emph{conjunctive queries}. They are commonly shortened as:
$$ \Query_1 = S(x), \mathit{CoA}(x,y).$$

A disjunction of conjunctive queries is known as a \emph{union of conjunctive queries} (UCQ). While they capture a rather complex set of queries, the algorithmic landscape of UCQs is remarkably well understood.

\begin{theorem} \citet{Dalvi2012TheDO} Let $\Query$ be a UCQ and $\pdb$ be a tuple-independent probabilistic database. Then the query $\Query$ is either:
  \begin{itemize}
  \item \emph{Safe}: $P_\pdb(\Query)$ can be computed in time polynomial in $|\pdb|$ for all probabilistic databases $\pdb$ using the standard lifted inference algorithm (see Section~\ref{par:queryeval});
  \item \emph{Unsafe}: Computing $P_\pdb(\Query)$ is a $\#P$-hard problem.
  \end{itemize}
  Furthermore, we can efficiently determine whether $\Query$ is safe or unsafe.
  
\end{theorem}

In much of the literature of probabilistic databases \citep{Suciu:2011:PD:2031527,Dalvi2012TheDO}, as well as throughout this paper, UCQs (and consequently conjunctive queries) are the primary query object studied.

\subsubsection{Reduction to Fully Quantified Boolean Queries}
In general, one is not always interested in computing fully quantified queries. For example, in Section~\ref{sec:eval} one of the queries we are interested in computing will be of the form
\begin{equation}
  \exists x, y. R(A,x) \land S(x,y) \land T(y,B)
\end{equation}

For relations $R,S,T$ and constants $A,B$. To convert this query to a fully quantified one, we need to \emph{shatter} the query \citep{VdBFTDB17}. In this case, we replace the binary relation $R(A,x)$ by the unary query $R_A(x)$, where $\forall x. R_A(x) = R(A,x)$. A similar procedure for $T$ gives us the following query:
\begin{equation}
  H_0 = \exists x,y. R_A(x) \land S(x,y) \land T_B(y)
\end{equation}

This is now a fully quantified query, and is also a simple example of an unsafe query. That is, for an arbitrary probabilistic database $\pdb$ we cannot compute $P_\pdb(Q)$ in time polynomial in $|\pdb|$ given our current independence and complexity assumptions.

\subsubsection{Efficient Query Evaluation}\label{par:queryeval}
In addition to providing an underlying probabilistic semantics, one of the motivations for exploring probabilistic databases as the formalism for relational embedding models was to be able to evaluate complex queries efficiently. Algorithm~\ref{alg:LiftR} does this in polynomial time for all safe queries. We now explain the steps in further detail. 

\renewcommand\algorithmicthen{}
\begin{algorithm}[tb]
\caption{$\bf Lift^R(\Query,\pdb)$, abbreviated by ${\bf L}(\Query)$ \label{alg:LiftR}}

  \begin{algorithmic}[1]
    \Require UCQ $\Query$ , prob.~database $\pdb$ with constants $T$.
    \Ensure The probability $P_{\pdb}(\Query)$
     \State {\textbf{Step 0}}  ~\emph{Base of Recursion}
     \Indent
     \If{\Query is a single ground atom $t$ }
                \If{$\ax{t:p} \in \pdb$} \Return $p$ \textbf{else} \textbf{return} $0$
       \EndIf 
    \EndIf  
    \EndIndent 
    \State {\textbf{Step 1}} ~\emph{Rewriting of Query}
     \Indent    
    \State Convert \Query to conjunction of UCQ: $\Query_{\land}\!\!=\!\Query_1   \land  \cdots  \land \Query_m$
    \EndIndent     
    \State \textbf{Step 2}  ~\emph{Decomposable Conjunction}    
    \Indent    
     \If{$m>1$  and $\Query_{\land}=\Query_1 \land \Query_2$ where $\Query_1 \perp \Query_2$}
         \State  \Return ${\bf L}(\Query_1) \cdot {\bf L}(\Query_2)$
    \EndIf
     \EndIndent     
  \State \textbf{Step 3}  ~\emph{Inclusion-Exclusion}    
    \Indent  
    \If{$m>1$ but $\Query_{\land}$ has no independent $\Query_i$}
    \State ~\emph{(Do Cancellations First)}
         \State \Return $\sum_{s \subseteq [m]}(-1)^{|s|+1} \cdot {\bf L}\left(\bigvee_{i \in s} \Query_i\right)$
          \EndIf
     \EndIndent   
%  \State \textbf{Step 4}  ~\emph{Rewriting of Query}    
%     \Indent  
%      \State Convert \query into CNF $\Query_{\text{CNF}}:=\Query_1 \wedge ... \wedge \Query_k$
%      \EndIndent  
    \State \textbf{Step 4}  ~\emph{Decomposable Disjunction}    
    \Indent    
     \If{$\Query=\Query_1 \vee \Query_2$ where $\Query_1 \perp \Query_2$}
     \State \Return $1- \left(1-{\bf L}(\Query_1)\right) \cdot \left(1-{\bf L}(\Query_2)\right)$
    \EndIf
     \EndIndent    
    \State \textbf{Step 5}  ~\emph{Decomposable Existential Quantifier}         
    \Indent    
     \If{\Query has a \emph{separator variable} $x$}
     \State \Return $1-\prod_{c \in T} \left(1-{\bf L}(\Query[x/c])\right)$ \label{line:separator}
    \EndIf
     \EndIndent   
    \State \textbf{Step 6}  ~\emph{Fail}  (the query is \#P-hard)     
  \end{algorithmic}
\end{algorithm} 

We begin with the assumption that $\Query$ has been processed to not contain any constant symbols, and that all variables appear in the same order in repeated predicate occurrences in $\Query$.  This can be done efficiently~\citep{Dalvi2012TheDO}.

\emph{Step 0} covers the base case where $\Query$ is simply a tuple, so it looks it up in $\pdb$. \emph{Step 1} attempts to rewrite the UCQ into a conjunction of UCQs to find decomposable parts. For example, the UCQ $(R(x) \land S(y,z)) \lor (S(x,y) \land T(x))$ can be written as the conjunction of $(R(x)) \lor (S(x,y) \land T(x))$ and $(S(y,z)) \lor (S(x,y) \land T(x))$. When multiple conjuncts are found this way, there are two options. If they are symbolically independent (share no symbols, denoted $\perp$), then \emph{Step 2} applies independence and recurses. Otherwise, \emph{Step 3} recurses using the inclusion-exclusion principle, performing cancellations first to maintain efficiency \citep{Dalvi2012TheDO}. If there is only a single UCQ after rewriting, \emph{Step 4} tries to split it into independent parts, applying independence and recursing if anything is found.

Next, \emph{Step 5} searches for a \emph{separator} variable, one which appears in every atom in $\Query$. If $x$ is a separator variable for $\Query$, and $a,b$ are different constants in the domain of $x$, this means that $\Query[x/a]$ and $\Query[x/b]$ are independent. This independence is again recursively exploited. Finally, if \emph{Step 6} is reached, then the algorithm has failed and the query provably cannot be computed efficiently \citep{Dalvi2012TheDO}, under standard complexity assumptions.

\section{RELATIONAL EMBEDDINGS AS PROBABILISTIC DATABASES} \label{sec:relpdb}
We now tackle the primary goal of this work: to use probabilistic databases as the formalism for doing probabilistic reasoning with relational embeddings. We begin with some background.

\subsection{RELATIONAL EMBEDDING MODELS}
Suppose we have a knowledge base $\kb$ consisting of triples $(h_i, R_i, t_i)$, denoting a head entity, relation, and tail entity (equivalently $R_i(h_i,t_i)$ in probabilistic database notation). Relational embedding models aim to learn continuous representations for both entities and relations, which together can be used to predict the presence of a triple. More formally:
\begin{definition} \label{def:embed}
  Suppose we have a knowledge base $\kb$ consisting of triples $(h_i, R_i, t_i)$, with entities $\ef$ and relations $\rf$. Then a \emph{relational embedding model} consists of
  \begin{itemize}
  \item Real vectors $v_R, v_e$ for all relations $R \in \rf$ and entities $e \in \ef$
    \item A scoring function $f(v_h,v_R,v_t) \to \mathbb{R}$ which induces a ranking over triples 
  \end{itemize}
\end{definition}
In general, these vectors may need to be reshaped into matrices or tensors before the scoring function can be applied. Table~\ref{tab:relscores} gives some examples of models with the form their vector representations take, as well as their scoring functions.

\begin{table*}[th]
  \caption{Example relational embedding scoring functions for $d$ dimensions}
  \centering
    \begin{tabular}{|c|c|c|c|}
    \hline
     Method & Entity Embedding & Relation Embedding & Triple Score \\
    \hline
    TransE~\citep{Bordes2013TranslatingEF} & $v_h, v_t \in \mathbb{R}^d$ & $v_R \in \mathbb{R}^d$ & $|| v_h + v_R - v_t ||$ \\
    \hline
    DistMult~\citep{Yang2014EmbeddingEA} &$v_h, v_t \in \mathbb{R}^d$& $v_R \in \mathbb{R}^d$ & $\langle v_h, v_R, v_t\rangle$  \\
    \hline
    Rescal~\citep{Nickel2011ATM} &$v_h, v_t \in \mathbb{R}^d$& $v_R \in \mathbb{R}^{d \times d}$ & $v^T_h v_R v_t$  \\
    \hline
    ComplEx~\citep{Trouillon2016ComplexEF} &$v_h, v_t \in \mathbb{C}^d$&$v_R \in \mathbb{C}^d$& Re($\langle v_h, v_R, \Bar{v_t} \rangle$) \\
    \hline
   \end{tabular}
  \label{tab:relscores}
\end{table*}

\subsection{PROBABILISTIC INTERPRETATIONS OF RELATIONAL EMBEDDINGS}
\begin{figure}[tb]
\vspace{-3mm}
\centering
\setlength{\tabcolsep}{3pt}
\scalebox{1.0}{
\begin{tabular}{l l  c}
\toprule
 \multicolumn{2}{c}{$R(x,y)$} & Score   \\
\midrule
$A$ & $B$ & -0.6 \\
$B$ & $C$ & 0.2 \\
$A$ & $C$ & 2.3 \\
\bottomrule
\end{tabular}
}
$\implies$
\setlength{\tabcolsep}{3pt}
\scalebox{1.0}{
\begin{tabular}{l l  c}
\toprule
 \multicolumn{2}{c}{$R(x,y)$} & $\Pr$ \\
\midrule
 $A$&$B$ & 0.35 \\
 $B$&$C$ & 0.55 \\
 $A$&$C$ & 0.91 \\
\bottomrule
\end{tabular}
}
\caption{An example of mapping a relational embedding to a probabilistic database using the sigmoid function.}
\label{fig:sc2prob}
\end{figure}
Given a relational embedding model from Definition~\ref{def:embed}, if we want to give it a clear probabilistic semantics using our knowledge of probabilistic databases from Section~\ref{sec:pdb}, we need to find a way to interpret the model as a probability distribution.

The simplest approach is to choose some mapping function $g:\mathbb{R} \to [0,1]$ which converts all the scores produced by the model's scoring function into probabilities. This provides us marginal probabilities, but no obvious joint distribution. Again, we can make the simplest choice and interpret these probabilities as being independent. That is, we can construct a probabilistic database where the probabilities are determined using our mapping function. Figure~\ref{fig:sc2prob} gives an example of such a conversion, using the sigmoid function as the mapping.

After doing this conversion, we can directly use Algorithm~\ref{alg:LiftR} to efficiently evaluate any safe query. This is a step in the right direction, but there are still two big issues here: firstly, as a simplifying assumption this triple-independence presents potential issues as discussed in \citet{Meilicke2019AnytimeBR}. For example, suppose we have a relational model containing \textit{Works-In(Alice, London)} and \textit{Lives-In(Alice, London)}: clearly these triples should not be independent. The second issue, which is perhaps even more critical for our purposes, is that even this assumption is not sufficient for all queries to be tractable:

\begin{theorem}\label{thm:relhard}
   Suppose we have a knowledge base $\kb$ with entities $\ef$ and relations $\rf$. Then, suppose we have a mapping function $g$ and a relational embedding model represented by a scoring function $f$ which is fully expressive. That is, for any configuration of marginal probabilities $P(R(h,t))$  over all possible triples, there is some assignment of entity and relation vectors such that $\forall R,h,t.\ \ g(f(v_h, v_R, v_t)) = P(R(h,t)) $.
   
   Then for any unsafe query $\Query$, evaluating $P(\Query)$ is a $\#P$-hard problem.
\end{theorem}

% So in addition to problematic independence assumptions, Theorem~\ref{thm:relhard} tells us that any sufficiently expressive relational embedding model using the mapping interpretation described above will certainly fail on any unsafe query. %So while this straightforward mapping approach does give us a probabilistic semantics and a way to evaluate queries, it still has issues with independence assumptions and query hardness.

\section{\TO}\label{sec:tractor}
The main takeaway from Section~\ref{sec:relpdb} is that although useful, interpreting relational embedding models as providing marginals for probabilistic databases still has major challenges. While we do now have a probabilistic semantics for our relational embedding model, the fact that we used the model as a black box means that we wind up treating all triples as independent.% This limits the expressiveness of the resulting probabilistic model, and comes with the issues previously discussed with treating all triples as independent. Furthermore, due to Theorem~\ref{thm:relhard} we will always have some limitations in our querying capabilities as a result of this black box usage.
The resulting expressiveness and tractability limitations motivate the search for a model which will not be treated as a black box by our probabilistic database semantics. Rather than simply having an arbitrary statistical model which fills in our probabilistic database, we would like to actually exploit properties of this statistical model. To put it another way: a fundamental underpinning of relational embedding models such as DistMult~\citep{Yang2014EmbeddingEA} or TransE~\citep{Bordes2013TranslatingEF} is that they make simplifying assumptions about how entity and relation vectors relate to link prediction. In Section~\ref{sec:relpdb}, our probabilistic interpretations of these models had no way of knowing about these simplifying assumptions: now we are going to express them in the language of PDBs.

\subsection{FACTORIZING IN PROBABILISTIC DATABASES}\label{subsec:fac}
Relational embedding models such as DistMult~\citep{Yang2014EmbeddingEA} and ComplEx~\citep{Trouillon2016ComplexEF}, or indeed any model derived from the canonical Polyadic decomposition~\citep{HitchcockCP1927} are built on an assumption about the way in which the tensor representing all triples factorizes. A similar idea has been used in the context of probabilistic first-order logic, where Boolean matrices representing binary relations are rewritten in terms of unary relations to make inference tractable \citep{Broeck2013OnTC}. We will now apply this technique of rewriting binary relations into unary relations as the basis for our relational embedding model.

Suppose we have a binary relation $R(x,y)$, and our model defines a single random variable $E(x)$ for each entity $x \in \ef$ as well as a random variable $T(R)$ for relation $R$. Then we assume that the relation $R$ decomposes in the following way:
\begin{equation} \label{eqn:reldecomp}
    \forall x,y. R(x,y) \iff E(x) \land T(R) \land E(y)
\end{equation}

We are assuming that all of the model's newly defined variables in $E$ and $T$ are independent random variables, so Equation~\ref{eqn:reldecomp} implies that
\begin{equation*}
    P(R(x,y)) = P(E(x)) \cdot P(T(R))\cdot P(E(y))
\end{equation*}

Figure~\ref{fig:1dex} gives an example of probabilities for $E$ and $T$, with corresponding probabilities for $R$ subject to Equation~\ref{eqn:reldecomp}. For example, we compute $P(R(A,B))$ by:
\begin{align*}
    P(R(A,B)) &= P(E(A)) \cdot P(T(R)) \cdot P(E(B)) \\
    % &= (0.2)(0.5)(0.4) \\
    &= 0.04
\end{align*}
To incorporate a relation $S$, we would define an additional $T(S)$ -- no new random variable per entity is needed. 

There are a few immediate takeaways from the rewrite presented in Equation~\ref{eqn:reldecomp}. Firstly, as a result of sharing dependencies in the model, we no longer have that all triples are independent of each other. For example $R(A,B)$ and $S(A,C)$ are not independent as they share a dependency on the random variable $E(A)$. Secondly, although these tuples are no longer independent (which would normally make query evaluation harder), their connection via new latent variables $E, T$ actually helps us. By assuming the latent $E, T$-tuples to be tuple independent, instead of the non-latent $R,S$-tuples, we are no longer subject to the querying limitations described by Theorem~\ref{thm:relhard}. In fact, \emph{any} UCQ can now be computed efficiently over the relations of interest. This will be proven in Section~\ref{subsec:qeval}, but intuitively binary relations must be involved for Algorithm~\ref{alg:LiftR} to get stuck, and our rewrite allows us to avoid this.

Of course, the major drawback is that Equation~\ref{eqn:reldecomp} describes an incredibly simple and inexpressive embedding model -- we can only associate a single probability with each entity and relation! We address this next.

\begin{figure}[tb]
\centering
\setlength{\tabcolsep}{3pt}
\scalebox{1}{
\begin{tabular}{l  c}
\toprule
 \multicolumn{1}{c}{$E(x)$} & $\Pr$   \\
\midrule
$A$ & 0.2 \\
$B$ & 0.4 \\
$C$ & 0.8 \\
\bottomrule
\end{tabular}
}
\setlength{\tabcolsep}{3pt}
\scalebox{1}{
\begin{tabular}{l  c}
\toprule
 \multicolumn{1}{c}{$T$} & $\Pr$ \\
\midrule
  $R$ & 0.5 \\
\bottomrule
\end{tabular}
}
{\large $\implies$}
\setlength{\tabcolsep}{3pt}
\scalebox{1}{
\begin{tabular}{l l  c}
\toprule
 \multicolumn{2}{c}{$R(x,y)$} & $\Pr$ \\
\midrule
 $A$&$B$ & 0.04 \\
 $B$&$C$ & 0.16 \\
 $A$&$C$ & 0.08 \\
\bottomrule
\end{tabular}
}
\caption{Example model tables $E, T_R$ and a few corresponding predictions for $R$}
\label{fig:1dex}
\vspace{-3mm}
\end{figure}

\subsection{MIXTURES \& \TO}
In a situation such as ours where we have a simple model which is efficient for some task but not expressive, the standard machine learning approach is to employ a mixture model. For example, while tree-shaped graphical models~\citep{Chow1968ApproximatingDP} provide efficient learning and inference, they are limited in their expressive capability: so a commonly used alternative is a mixture of such models~\citep{Meila1998LearningWM}. Similarly, while Gaussians are limited in their expressiveness, mixture of Gaussian models~\citep{titterington_85} have found widespread use throughout machine learning. These mixtures can typically approximate any distribution given enough components.

In our case, we will take the model described in Equation~\ref{eqn:reldecomp} as our building block, and use it to create \TO.

\begin{definition}
  \TOs with $d$ dimensions is a mixture of $d$ models each constructed from Equation~\ref{eqn:reldecomp}. That is, it has tables $T_i, E_i$ analagous to $T$ and $E$ above for each element $i$ of the mixture.  Then, for each element $i$ we have
  \begin{equation*}
    \forall x,y. R_i(x,y) \iff E_i(x) \land T_i(R) \land E_i(y)
  \end{equation*}
  The probability of any query is then given by \TOs as the average of the probabilities of the $d$ mixture components. 
\end{definition}

Figure~\ref{fig:2dex} gives an example 2-dimensional \TOs model, including probabilities for $E_1,E_2,T_1,T_2$, and corresponding probabilities for materialized relation $R$. For example, we compute $P(R(A,B))$ by:
\begin{align*}
    P(R(A,B)) &= \frac{1}{2}(P(E_1(A))\cdot P(T_1(R))\cdot P(E_1(B)) \\
    &+ P(E_2(A))\cdot P(T_2(R))\cdot P(E_2(B))) \\
    % &= \frac{1}{2} ((0.2)(0.5)(0.4) + (0.6)(1)(0.5)) \\
    &= 0.17
\end{align*}

We see that the components of the mixture form what we typically think of as dimensions of the vectors of embeddings. For example, in Figure~\ref{fig:2dex} the embedding of entity $A$ is $(E_1(A), E_2(A)) = (0.2,0.6)$.

% Defining \TOs in this way means that we have a model which is sufficiently expressive for predicting facts, while also allowing for efficient inference across a wide range of queries due to the structure of Equation~\ref{eqn:reldecomp}.

\begin{figure*}[tb]
\vspace{-3mm}
\centering
\setlength{\tabcolsep}{3pt}
\scalebox{1}{
\begin{tabular}{l  c}
\toprule
 \multicolumn{1}{c}{$E_1(x)$} & $\Pr$   \\
\midrule
$A$ & 0.2 \\
$B$ & 0.4 \\
$C$ & 0.8 \\
\bottomrule
\end{tabular}
}
\setlength{\tabcolsep}{3pt}
\scalebox{1}{
\begin{tabular}{l  c}
\toprule
 \multicolumn{1}{c}{$T_1$} & $\Pr$ \\
\midrule
  $R$& 0.5 \\
\bottomrule
\end{tabular}
}
{\large $\ +$}
\setlength{\tabcolsep}{3pt}
\scalebox{1}{
\begin{tabular}{l  c}
\toprule
 \multicolumn{1}{c}{$E_2(x)$} & $\Pr$   \\
\midrule
$A$ & 0.6 \\
$B$ & 0.5 \\
$C$ & 0.2 \\
\bottomrule
\end{tabular}
}
\setlength{\tabcolsep}{3pt}
\scalebox{1}{
\begin{tabular}{l c}
\toprule
 \multicolumn{1}{c}{$T_2$} & $\Pr$ \\
\midrule
  $R$& 1 \\
\bottomrule
\end{tabular}
}
{\large $\implies$}
\setlength{\tabcolsep}{3pt}
\scalebox{1}{
\begin{tabular}{l l  c}
\toprule
 \multicolumn{2}{c}{$R(x,y)$} & $\Pr$ \\
\midrule
 $A$&$B$ & 0.17 \\
 $B$&$C$ & 0.13 \\
 $A$&$C$ & 0.10 \\
\bottomrule
\end{tabular}
}
\caption{Example \TOs model tables $E_1, E_2, T_1, T_2$ and a few corresponding predictions for $R$}
\label{fig:2dex}
\vspace{-3mm}
\end{figure*}
\subsection{EQUIVALENCE TO DISTMULT} \label{subsec:equiv}
The first question we need to ask about \TOs is how effective it is for link prediction.

\begin{theorem}\label{thm:equiv}
 Suppose we have entity embeddings $v_h, v_r \in \mathbb{R}^d$ and relation embedding $v_R \in \mathbb{R}^d$. Then \TOs and DistMult will assign identical scores (within a constant factor) to the triple $(h, R, t)$ (equivalently $R(h,t)$).
\end{theorem}

We already know from~\citet{Yang2014EmbeddingEA} that DistMult is effective for link prediction, so \TOs must also be.

\subsubsection{Positive and Negative Weights}
While we have seen that the computation used for link prediction in \TOs is identical to that of DistMult, there remains a key difference: \TOs has a probabilistic semantics, and thus all parameters must be probabilities. One option here is to indeed force all parameters to be positive, and live with any performance loss incurred. Another option is allowing for negative probabilities in $E, T$ meaning that we can achieve exactly the same link prediction results as DistMult, whose predictive power is well documented~\citep{Yang2014EmbeddingEA}. It has been previously shown that probability theory can be consistently extended to negative probabilities~\citep{bartlett_1945}, and their usefulness has also been documented in the context of probabilistic databases~\citep{Jha2012ProbabilisticDW, VdBKR14}. Furthermore, by adding a simple disjunctive bias term, we can ensure that all fact predictions are indeed positive probabilities. In Section~\ref{sec:eval} we will explore both options.

\subsection{QUERY EVALUATION} \label{subsec:qeval}
Finally, we explore query evaluation for the \TOs model. Suppose we have some arbitrary UCQ $\Query$ over binary and unary relations, and we would like to compute $P(\Query)$ where all binary relations are given by a \TOs model. First, we substitute each binary relation according to Equation~\ref{eqn:reldecomp} using \TOs tables $E$ and $T$. What remains is a query $\Query'$ which contains only unary relations. 

\begin{theorem}\label{thm:unary}
   Suppose that $\Query'$ is a UCQ consisting only of unary relations. Then $\Query'$ is safe. 
\end{theorem}
\begin{proof}
We prove this by showing that Algorithm~\ref{alg:LiftR} never fails on $\Query'$. Consider if $\Query'$ cannot be rewritten as a conjunction of UCQs. Then each CQ must contain only a single quantified variable, or else that CQ would contain 2 separate connected components (due to all relations unary). Thus, if we ever reach Step 5 of Algorithm~\ref{alg:LiftR}, each CQ must have a separator. So $\Query'$ is safe.
% \begin{enumerate}
%     \item For each CQ in $\Query'$, each relation can only appear once. This is because the relations are all unary, and thus any repeated relations will imply each other. For example $\exists x,y. R(x)\land R(y)$ is equivalent to $\exists x. R(x)$.
%     \item If $\Query'$ can not be be rewritten as a conjunction of UCQs, each CQ must contain a single quantified variable. This is because if a CQ does contain more than one variable, observation 1 tells us there must be a way to split the CQ, which would allow us to rewrite the UCQ as a conjunction of multiple UCQs.
% \end{enumerate}

% These two observations tell us that if we ever reach step 5 of Algorithm~\ref{alg:LiftR}, each CQ has only has a single variable, meaning we must have a separator. Thus Algorithm~\ref{alg:LiftR} never fails and $\Query'$ must be safe.

\end{proof}

\section{EMPIRICAL EVALUATION} \label{sec:eval}
We will now empirically investigate the effectiveness of \TOs as a relational embedding model. As discussed in Section~\ref{subsec:equiv}, for the purposes of link prediction \TOs actually turns out to be equivalent to DistMult. While it does have certain limitations regarding asymmetric relations, the overall effectiveness of DistMult for link prediction has been well documented~\citep{Yang2014EmbeddingEA}, so we will not be evaluating \TOs on link prediction. Instead, we will focus on evaluating \TO's performance when computing more advanced queries.\footnote{Code is available at \url{https://github.com/ucla-starai/pdbmeetskge}} While training the models we evaluated, we confirmed that training \TOs and DistMult produced the same embeddings and link prediction performance.

\subsection{QUERIES \& COMPARISON TARGET}
As our comparison for evaluation, we will use the graph query embeddings (GQE)~\citep{Hamilton2018EmbeddingLQ} framework and evaluation scheme. Fundamentally, GQE differs from \TOs in its approach to query prediction. Where \TOs is a distribution representing beliefs about the world which can then be queried to produce predictions, GQE treats queries as their own separate prediction task and defines vector operations to specifically be used for conjuctive query prediction. The consequence of this is that where \TOs has a single correct way to answer any query (the answer induced by the probability distribution), a method in the style of GQE needs to find a new set of linear algebra tools for each type of query. 

In particular, GQE uses geometric transformations as representations for conjunction and existential quantifiers, allowing it to do query prediction via repeated application of these geometric transformations. \citet{Hamilton2018EmbeddingLQ} detail further exactly which queries are supported, but put simply it is any conjunctive query that can be represented as a directed acyclic graph with a single sink.

To evaluate these models, the first question is which queries should be tested. We describe a query template as follows: $R,S,T$ are placeholder relations, $A,B,C$ placeholder constants, $x,y,z$ quantified variables, and $t$ is the parameterized variable. That is, the goal of the query is to find the entity $t$ which best satisfies the query (in our framework gives the highest probability). Table~\ref{tab:queries} gives a series of example template CQs and UCQs. In Figure~\ref{fig:query_cat}, we categorize each of these query templates based on their hardness under standard probabilistic database semantics, as well as their compatibility with GQE. Notice that \TOs can compute all queries in Figure~\ref{fig:query_cat} in time linear in the domain size, including queries $\Query_4, \Query_{11}, \Query_{10}$ which would be $\#P$-hard in a standard tuple-independent probabilistic database. For the sake of comparison, we perform our empirical evaluation using the queries that are also supported by GQE.

\begin{table}[t] \label{tab:queries}
    \caption{Example CQs and UCQs}
\renewcommand{\arraystretch}{1.25}
    \centering
    \begin{tabular}{rl}
    \hline
    \hline
         $\Query_1(t) =$ &$ R(A,t)$ \\
         $\Query_2(t) =$ &$\exists x. R(A,x)$ \\
         $\Query_3(t) = $ &$\exists x. R(A,x) \land S(x,t)$ \\
         $\Query_4(t) =$ &$ \exists x,y. R(A,x) \land S(x,y) \land T(y,t)$ \\
         $\Query_5(t) =$ &$ R(A,t) \land S(B,t)$ \\
         $\Query_6(t) =$ &$ R(A,t) \land S(B,t) \land T(C,t)$ \\
         $\Query_7(t) =$ &$ \exists x. R(A,x) \land S(x,t)$ \\
         &$\lor \exists y. R(A,y) \land T(y,t)$ \\
         $\Query_8(t) =$ &$ \exists x. R(A,x) \land S(x,t) \land T(B,t)$ \\
         $\Query_9(t) =$ &$ \exists x. R(A,x) \land S(B,x) \land T(x,t)$ \\
         $\Query_{10}(t) =$ &$ \exists x_1,y_1. R(A,x_1) \land S(x_1,y_1)$ \\ 
         &$\lor \exists x_2,y_2.  S(x_2,y_2) \land T(y_2, t)$ \\
         $\Query_{11}(t) =$ &$ \exists x,y,z. R(A,x) \land S(x,y) \land T(y,z)$ \\
        %  $R(a,t) \land S(b,t) \land Q(c,t)$ & \cmark & \cmark \\
        \hline 
        \hline
    \end{tabular}
\end{table}

\subsection{DATASET}
For our dataset, we use the same choice in relational data as \citet{Hamilton2018EmbeddingLQ}. In that work, two datsets were evaluated on, which were termed bio and reddit respectively. Bio is a dataset consisting of knowledge from public biomedical databases, consisting of nodes which correspond to drugs, diseases, proteins, side effects, and biological processes. It includes 42 different relations, and the graph in total has over 8 million edges between 97,000 entities. The reddit dataset was not made publicly available so we were unable to use it for evaluation. 

\subsubsection{GENERATING QUERIES}
While the bio dataset provides our entities and relations, we need to create a dataset of conjunctive queries to evaluate on. For this, we again follow the procedures from \citet{Hamilton2018EmbeddingLQ}. First, we sample a 90/10 train/test split for the edges in the bio data. Then, we generate evaluation queries (along with answers) using both train and test edges from the bio dataset, but sample in such a way that each test query relies on at least one edge not present in the training data. This ensures that we can not template match queries based on the training data. For each query template we sample 10,000 queries for evaluation. For further details, including queries for which some edges are adversarially chosen, see \citet{Hamilton2018EmbeddingLQ}.

As an example, templating on $\Query_4$ can produce:
\begin{align*}
  D_? \exists p_1 \exists p_2 & \textsc{Activates}(P_3,p_2) \land \textsc{Catalyzes}(p_2,p_1) \\
                           & \land \textsc{Target}(p_1,D)
\end{align*}

where $D$ is the drug we would like to find and $p_1,p_2,P_3$ are proteins.
\begin{figure}[tb]
    \centering
    \includegraphics[width=\columnwidth]{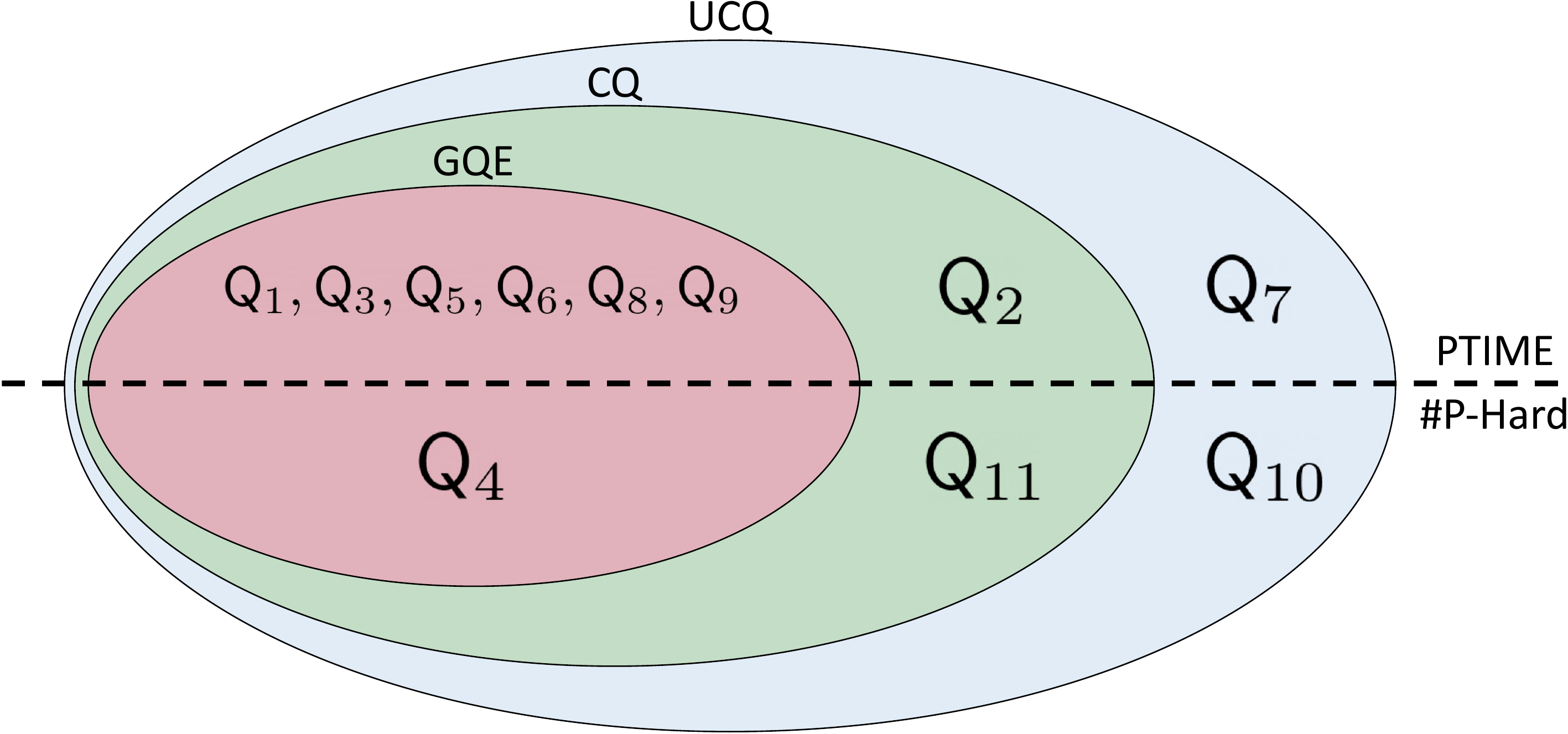} 
    \vspace{-2mm}
    \caption{Categorizing different queries based on safeness and compatibility with GQE \citep{Hamilton2018EmbeddingLQ}. \TOs efficiently supports all queries in the diagram.}
    \label{fig:query_cat}
    \vspace{-3mm}
\end{figure}
\subsection{EVALUATION}
For each evaluation query, we ask the model being evaluated to rank the entity which answers the query in comparison to other entities which do not. We then evaluate the performance of this ranking using a ROC AUC score, as well as an average percentile rank (APR) over 1000 random negative examples.
\subsubsection{Baselines and Model Variants}
We evaluate two versions of our model: \TOs indicates a model where the unary predicate probabilities are allowed to be negative, and a bias term is added to ensure all triples have positive predicted probability. \TOPs indicates a model where unary predicate probabilities are constrained to be positive via squaring. 

As baselines, we consider model variants from \citet{Hamilton2018EmbeddingLQ} that do not include extra parameters that must be trained on queries, as our model contains no such parameters. These models are each built on an existing relational embedding model (Bilinear~\citep{Nickel2011ATM}, DistMult~\citep{Yang2014EmbeddingEA}, and TransE~\citep{Bordes2013TranslatingEF} respectively) used for link prediction and composition, as well as a mean vector operator used for queries. For example, for query $\Query_5$, these baselines will make a prediction for $t$ that satisfy $R(a,t)$ and $S(b,t)$ separately, and then take the mean of the resulting vectors.

\begin{table}[t]
    \caption{Overall query performance on bio dataset}
\renewcommand{\arraystretch}{1.25}
    \centering
    \begin{tabular}{ccc}
        Method & AUC & APR \\
        \hline \hline 
        Bilinear & 79.2 & 78.6 \\
        DistMult & 86.7 & 87.5 \\
        TransE & 78.3 & 81.6 \\
        \hline
        \TOP & 75.0 & 84.5 \\
        \TO & 82.8 & 86.3 \\
        %  & \emph{TractOR}+ & \emph{TractOR} & Bilinear & DistMult & TransE \\
        %  AUC& 75.0 & 82.8 & 79.2 & 86.7 & 78.3\\
        %  APR& 84.5 & 86.3 & 78.6 & 87.5 & 81.6
    \end{tabular}
    \label{tab:res}
    \vspace{-3mm}
\end{table}

\subsubsection{Training}
All model variants and baselines were trained using the max-margin approach with negative sampling \citep{Mikolov2013DistributedRO} which has become standard for training relational embedding models~\citep{Nickel2015ARO}. Parameters were optimized using the Adam optimizer~\citep{Kingma2014AdamAM}, with an embedding dimension of 128, a batch size of 256, and learning rate of $0.01$.

\subsubsection{Results \& Discussion}
Table~\ref{tab:res} presents AUC and APR scores for all model variants and baselines on the bio dataset. \TOs and \TOPs both perform better than TransE and Bilinear based baselines in APR, and are competitive with the DistMult baseline. Evaluating by AUC the performance is slightly worse, but \TOs remains better than or comparable to all baselines. These results are very encouraging as \TOs is competitive despite the fact that it is representing much more than just conjunctive query prediction. \TOs represents a complete probability distribution: effective and efficient query prediction is simply a direct consequence.

Another interesting observation to make here is the gap between \TOs and \TOP, where the only difference is whether the parameters are constrained to be positive. The difference in performance here essentially comes down to the difference in performance on link prediction: not being allowed to use negative values makes the model both less expressive and more difficult to train, leading to worse performance on link prediction. We did not find that increasing the number of dimensions used in the representation to make up for not having negative values helped significantly. Finding ways to improve link prediction subject to this constraint seems to be valuable for improving performance on query prediction.

% \begin{algorithm}[tb]
% \caption{$\bf Eval(\Query,\mf)$ \label{alg:eval}}

%   \begin{algorithmic}[1]
%     \Require UCQ $\Query$ , TractOR model $\mf$
%     \Ensure The probability $P_{\mf}(\Query)$
          
%   \end{algorithmic}
% \end{algorithm}

\section{DISCUSSION \& RELATED WORK}\label{sec:disc}
\paragraph{Querying Relational Embeddings}
Previous work studying queries beyond link prediction in relational embedding models proposed to replace logical operators with geometric transformations~\citep{Hamilton2018EmbeddingLQ}, and learning new relations representing joins~\citep{Krompass2014QueryingFP}. Our work differs from these in that we formalize an underlying probabilistic framework which defines algorithms for doing querying, rather than treating querying as a new learning task.
\paragraph{Symmetric Relations}
A limitation of the \TOs model which also appears in models like DistMult \citep{Yang2014EmbeddingEA} and TransE \citep{Bordes2013TranslatingEF} is that since head and tail entities are treated the same way, they can only represent symmetric relations. This is, of course, problematic as many relations we encounter in the wild are not. Solutions to this include assigning complex numbers for embeddings with an asymmetric scoring function~\citep{Trouillon2016ComplexEF}, and keeping separate head and tail representations but using inverse relations to train them jointly~\citep{Kazemi2018SimplEEF}. Borrowing these techniques presents a straightforward way to extend \TOs to represent asymmetric relations.

\paragraph{Probabilistic Training}
One potential disconnect in \TOs is that while it is a probabilistic model, it is not trained in a probabilistic way. That is, it is trained in the standard fashion for relational embedding models using negative sampling and a max-margin loss. Other training methods for these models such as cross-entropy losses exist and can improve performance~\citep{Ruffinelli2020You} while being more probabilistic in nature. In a similar vein, \citet{Tabacof2020Probability} empirically calibrates probabilities to be meaingful with respect to the data.  An interesting open question is if \TOs can be trained directly using a likelihood derived from its PDB semantics.

\paragraph{Incomplete Knowledge Bases}
One of the main goals of this work is to overcome the common issue of \emph{incomplete} knowledge. That is, what do we do when no probability at all is known for some fact. In this work, we directly incorporate machine learning models to overcome this. Another approach to this problem is to suppose a range of possibilities for our unknown probabilities, and reason over those. This is implemented via open-world probabilistic databases~\citep{CeylanKR16}, with extensions to incorporate background information in the form of ontological knowledge~\citep{DBLP:conf/aaai/BorgwardtCL19} and summary statistics~\citep{FriedmanIJCAI19}.

% \paragraph{Background Knowledge}

\paragraph{Increasing Model Complexity}
\TOs is a mixture of very simple models. While this makes for highly efficient querying, accuracy could potentially be improved by rolling more of the complexity into each individual model at the PDB level. The natural approach to this is to follow \citet{Broeck2013OnTC} and replace our simple unary conjunction with a disjunction of conjunctions. This raises interesting theoretical and algorithmic questions with potential for improving query prediction.

\paragraph{Further Queries}
Finally, there are further question one can ask of a PDB beyond the probability of a query. For example, \citet{gribkoff2014most} poses the question of which world (i.e. configuration of tuple truths) is most likely given a PDB and some constraints, while \citet{Ceylan2017MostPE} studies the question of which explanations are most probable for a certain PDB query being true. Extending these problems to the realm of relational embeddings poses many interesting questions.

% \tal{Talk about calibration of probabilities for doing tasks beyond ranking? Maybe also about calibrating probabilities across different relations which isn't done for max-margin based learning}
% \tal{Talk about how we can make the model more complex without immediately jumping to mixtures using disjunction}
% \tal{Incoroporating background knowledge?}
\subsubsection*{Acknowledgements}
We thank Yitao Liang, YooJung Choi, Pasha Khosravi, Dan Suciu, and Pasquale Minervini for helpful feedback and discussion. This work is partially supported by NSF grants \#IIS-1943641, \#IIS-1633857, \#CCF-1837129, DARPA XAI grant \#N66001-17-2-4032, and gifts from Intel and Facebook Research.

\small
\bibliographystyle{plainnat}
\bibliography{ref}

\end{document}